\documentclass[11pt]{article}

\usepackage[preprint]{acl}

\usepackage{times}
\usepackage{latexsym}
\usepackage[T1]{fontenc}
\usepackage[utf8]{inputenc}
\usepackage{microtype}
\usepackage{graphicx}
\usepackage{float}
\usepackage{placeins}
\usepackage{booktabs}
\usepackage{multirow}
\usepackage[table]{xcolor}
\usepackage{amsfonts}
\usepackage{nicefrac}
\usepackage{amsmath}
\usepackage{amssymb}
\usepackage{amsthm}
\newcommand{\vpara}[1]{\vspace{0.05in}\noindent \textbf{#1 }}
\newtheorem{theorem}{Theorem}
\newtheorem{corollary}{Corollary}

\title{Dynamic Rollout Editing for Reducing Overthinking in RL-Trained Reasoning Models}

\author{
  \textbf{Zihao Wei}\textsuperscript{1,2},
  \textbf{Wenjie Shi},
  \textbf{Liang Pang}\textsuperscript{1}\Thanks{Corresponding author.},
  \textbf{Jingcheng Deng}\textsuperscript{1,2},
  \textbf{Shicheng Xu}\textsuperscript{1,2},
  \textbf{Shasha Guo}\textsuperscript{1},\\
  \textbf{Zenghao Duan}\textsuperscript{1,2},
  \textbf{Jiahao Liu},
  \textbf{Jingang Wang},
  \textbf{Huawei Shen}\textsuperscript{1},
  \textbf{Xueqi Cheng}\textsuperscript{1} \\
  \textsuperscript{1}Institute of Computing Technology, Chinese Academy of Sciences, Beijing, China \\
  \textsuperscript{2}University of Chinese Academy of Sciences, Beijing, China \\
  \texttt{\{weizihao22z, pangliang\}@ict.ac.cn} \\
  \texttt{shiwj09@163.com}
}

\begin{document}
\maketitle

\begin{abstract}

Long-form chain-of-thought reasoning can improve LLM performance on complex tasks, but models often continue generating unnecessary reasoning after a correct answer has emerged. We refer to this behavior as overthinking. We study this phenomenon from the perspective of GRPO-style reinforcement learning (RL) post-training, framing it as a training-time credit-assignment problem rather than merely a decoding-time stopping problem.
In rollouts sampled at the onset of GRPO training, we observe that successful trajectories can exhibit a slightly higher degree of overthinking than unsuccessful trajectories for the same prompts. This early imbalance provides a starting point for an undesirable feedback loop: because GRPO assigns sequence-level credit, it cannot distinguish the solution-reaching prefix from the unnecessary continuation that lengthens a successful trajectory. Both receive positive update signal, allowing the initial imbalance to grow into more severe overthinking during training.
To address this issue, we introduce Dynamic Rollout Editing (DRE), a training-time intervention for successful trajectories that continue thinking after answer emergence. DRE preserves the accepted verified prefix, edits the remaining thinking, and prefers the edited trajectory within the same RL group, weakening the preference signal for unnecessary thinking without penalizing the reasoning needed to reach the answer. Experiments across diverse tasks show  the effectiveness of DRE.

\end{abstract}

\section{Introduction}

\begin{figure}[t]
\centering
\includegraphics[width=0.95\linewidth]{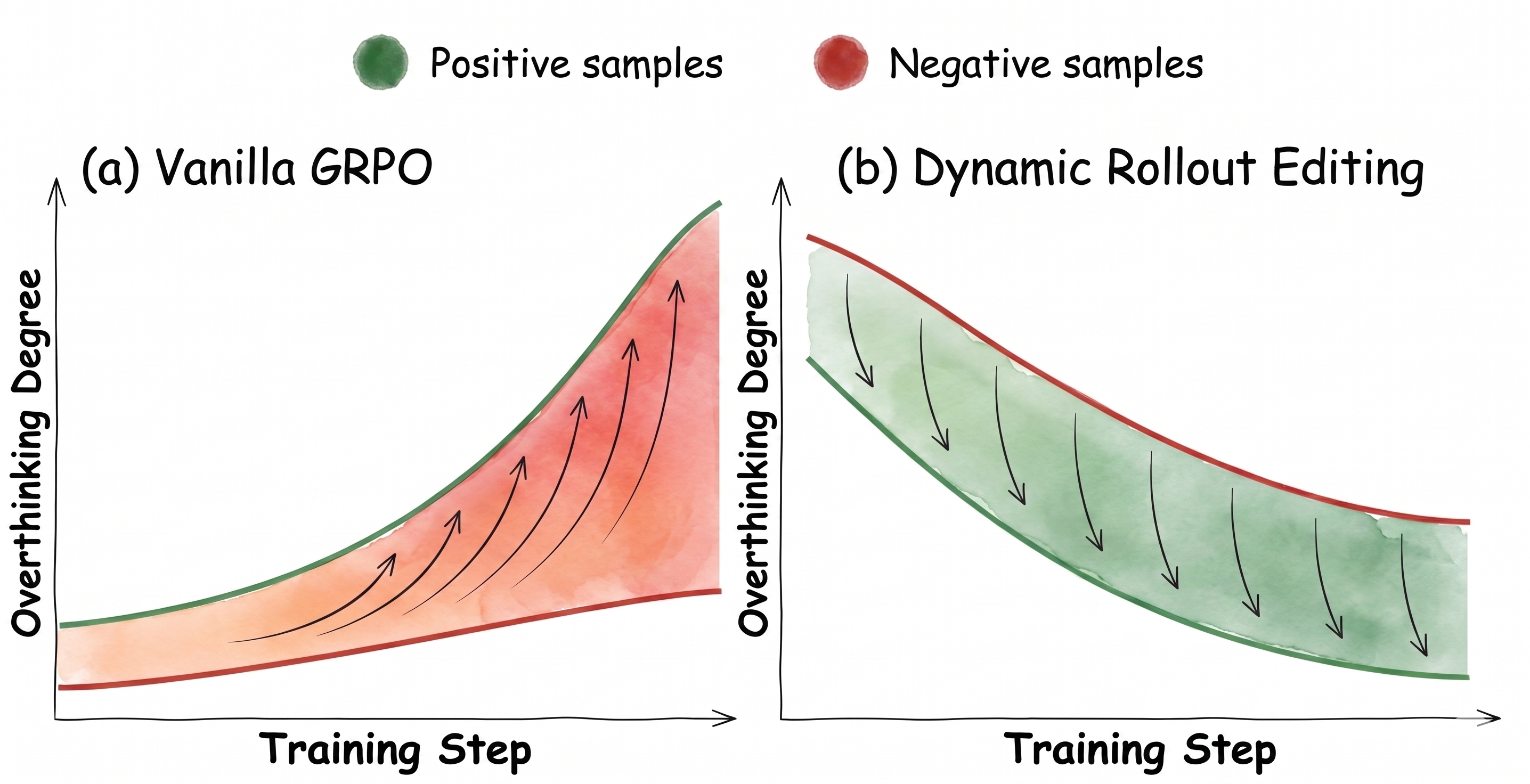}
\caption{Schematic view of training-time overthinking. (a) Vanilla GRPO attaches positive credit to successful trajectories with greater overthinking. (b) Dynamic rollout editing prefers verified shorter edits, reducing the association between success and overthinking.}
\label{fig:causal_pipeline}
\vspace{-12pt}
\end{figure}

Large language models (LLMs) have substantially improved reasoning performance by generating long-form chain-of-thought processes before producing final responses~\cite{DBLP:journals/corr/abs-2501-12948, yang2025qwen3technicalreport, DBLP:journals/corr/abs-2509-18883}. This paradigm is central to reinforcement learning post-training, where models explore reasoning trajectories and learn from outcome feedback. However, longer thinking is not always better: generation length and answer quality are not monotonically related~\cite{ghosal2025doesthinkinghelpmirage, DBLP:journals/corr/abs-2509-23392}. Once a correct answer has appeared, further thinking can increase inference cost and may even move the model away from the correct solution, a phenomenon commonly described as \emph{overthinking}~\cite{DBLP:journals/corr/abs-2503-16419, DBLP:journals/corr/abs-2503-24377}.

Existing work on overthinking largely treats it as a problem of controlling reasoning length: deciding when the model should stop or how its thinking process can be shortened without hurting accuracy. Studies of reasoning dynamics identify boundaries such as the reasoning completion point, where useful reasoning gives way to unnecessary thinking~\cite{DBLP:journals/corr/abs-2502-07266, DBLP:journals/corr/abs-2505-00127, wei2025evolutionthought}. Building on this view, early-exit decoding, RL-based stopping controllers, and length-oriented regularization aim to reduce excessive generation~\cite{DBLP:journals/corr/abs-2504-15895, DBLP:journals/corr/abs-2505-07686, DBLP:journals/corr/abs-2501-12599}. While useful, these methods mainly intervene after the policy has learned its reasoning behavior. 
They leave open how reinforcement learning may reinforce overthinking after a correct answer is already verifiable.

We study this issue in Group Relative Policy Optimization (GRPO)~\cite{DBLP:journals/corr/abs-2402-03300}, a widely used post-training method for reasoning models. GRPO assigns sequence-level rewards and updates the policy using group-relative advantages. This design improves task success, but can also create a credit-assignment problem: a successful trajectory may contain both thinking that leads to the correct answer and unnecessary thinking after the answer has emerged.
Since the same positive advantage is applied to the trajectory as a whole, GRPO does not distinguish the reasoning prefix that leads to the verified answer from the later overthinking that only lengthens the trajectory.

To examine this credit-assignment problem, we make answer emergence observable through a parser-verifier analysis. For each trajectory, we identify the first point at which a correct answer can be extracted and externally verified, allowing us to compare behavior before and after answer emergence. 
In rollouts sampled early in GRPO training, successful trajectories already show a slight tendency to contain more overthinking after answer emergence than unsuccessful trajectories sampled for the same prompts. This initial asymmetry can be amplified by sequence-level updates: when unnecessary thinking is part of a successful trajectory, GRPO may reinforce it together with the thinking that leads to the answer.
As training improves task success, this coupled credit can make overthinking after answer emergence more persistent, including redundant checking, answer revisits, and repeated verification.
This observation motivates a different intervention. Rather than only stopping a trained model earlier, we aim to change the training signal so that task success is less coupled with overthinking after answer emergence. A simple length penalty may suppress useful thinking, while penalizing an entire successful trajectory may weaken the prefix needed to solve the problem. An effective intervention should preserve the verified thinking prefix while reducing the preference signal assigned to subsequent overthinking.

In this work, we introduce \textbf{dynamic rollout editing} (DRE), a training-time intervention for successful trajectories that contain overthinking after answer emergence. Figure~\ref{fig:causal_pipeline} contrasts vanilla GRPO with DRE: vanilla GRPO can reinforce successful trajectories that include overthinking after answer emergence, whereas DRE shifts preference toward verified shorter counterparts.
For a successful trajectory with overthinking after answer emergence, DRE constructs an edited counterpart that preserves the verified prefix up to a conservative boundary and replaces the subsequent overthinking. 
The edited trajectory is then preferred within the same GRPO group, allowing the model to retain a positive training signal for reaching the correct answer while weakening the association between success and overthinking. We further introduce two complementary training strategies: prefix masking avoids directly penalizing the shared verified prefix, while GClip provides a learning signal for edited tokens near the edit point.

\textbf{Contributions.} 
(1) We identify overthinking after answer emergence as a training-time credit-assignment problem in GRPO-style post-training. Our analysis shows that sequence-level advantages can mix credit between reasoning that leads to the answer and unnecessary thinking after answer emergence.
(2) We propose DRE, a training-time intervention that weakens the association between task success and unnecessary thinking after answer emergence. DRE assigns a higher preference to shorter edited counterparts within the same GRPO group while preserving the verified reasoning prefix that reaches the answer.
(3) We demonstrate the broad effectiveness of DRE across mathematical reasoning, science QA, and code generation. DRE reduces thinking length while preserving task performance, mainly after answer emergence.
\section{Overthinking After Answer Emergence}
\label{sec:pathology}

\begin{figure*}[t]
\centering
\includegraphics[width=0.995\linewidth]{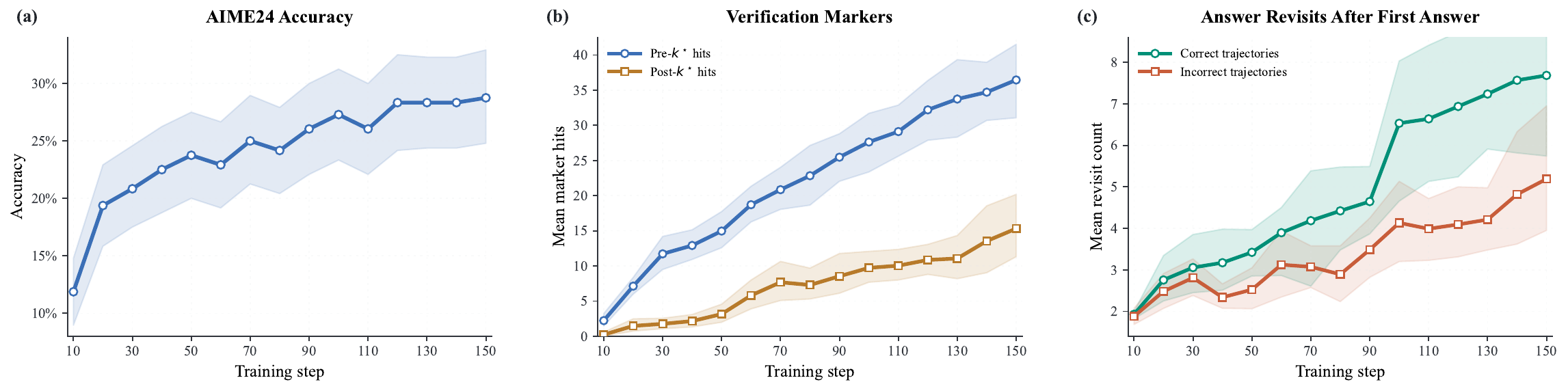}
\caption{Behavioral evidence from an R1-Zero-like GRPO run of Qwen3-4B-Base. \textbf{(a)} AIME24 accuracy improves during training. \textbf{(b)} Verification and checking keywords increase both before and after answer emergence. \textbf{(c)} For same prompts, successful rollouts show more answer revisits than unsuccessful rollouts.}
\label{fig:confound}
\vspace{-12pt}
\end{figure*}

To examine this phenomenon, this section first formulates the generation process and defines answer emergence as the earliest point at which a verified answer can be extracted from a trajectory.
This point provides a concrete boundary for analysis: tokens before it may be needed to reach the answer, whereas tokens after it allow us to measure overthinking once the answer has already appeared. We then present preliminary observations showing that, for the same prompts, successful rollouts often contain more overthinking after answer emergence than unsuccessful rollouts. Finally, we analyze how GRPO can reinforce this behavior: because rewards are assigned at the sequence level, the update does not explicitly distinguish reasoning that leads to the answer from overthinking that follows it when both occur in a successful trajectory.

\subsection{Problem Formulation}

We formalize overthinking after answer emergence in terms of thinking trajectories. Given a prompt $x$, the model generates a thinking process $t$ followed by a content process $y$. We model this generation process as \(\pi_\theta(t, y \mid x)=\pi_\theta(t \mid x)\pi_\theta(y \mid x, t)\). The resulting trajectory is denoted as \(\tau = (x, \texttt{<think>}~t~\texttt{</think>}~y)\).

For each trajectory, we split the thinking process into sentences and scan sentence prefixes from left to right. At each prefix, a parser extracts a candidate answer, and a verifier checks whether the candidate is correct for the given prompt. We define \(k^\star(\tau)\) as the first sentence index whose prefix passes this verification check. If no prefix contains a verified answer, we set \(k^\star(\tau)=K_t(\tau)\), where \(K_t(\tau)\) is the number of thinking sentences in \(\tau\).

Thus, \(k^\star(\tau)\) marks the earliest sentence-level point at which a correct answer becomes verifiable in the thinking process. We use this point as an analysis boundary for measuring overthinking after answer emergence, without assuming that every subsequent token is redundant.

\subsection{Preliminary Observations}

\paragraph{Setting.} 
We train Qwen3-4B-Base with GRPO following the DeepSeek-R1-Zero~\cite{DBLP:journals/corr/abs-2501-12948} setting, using only rule-based rewards for answer correctness and format. This setup allows us to examine whether overthinking after answer emergence increases during RL training.

For each rollout in which a verified answer appears, \(k^\star(\tau)\) denotes the first sentence index at which the prefix contains that answer.
We count keywords related to verification and checking in the sentences before and after \(k^\star(\tau)\).
Keywords in sentences before \(k^\star(\tau)\) may reflect useful checking before the answer is reached, whereas keywords in sentences after \(k^\star(\tau)\) serve as indicators of possible overthinking, since the model continues to check or verify after the answer has already appeared. We then examine whether both quantities increase during training.

To compare successful and unsuccessful rollouts for the same prompt, we use answer revisits as a rollout-level measure, because unsuccessful rollouts may not contain a valid \(k^\star(\tau)\). For each rollout with an extractable final answer, we count how often that answer is repeated after its first occurrence, and compare this count across rollouts sampled for the same prompt.

\textbf{Observation: RL training improves accuracy while amplifying overthinking behaviors such as answer repetition and redundant verification.}
Figure~\ref{fig:confound} (a) illustrates the accuracy gain during training. Figure~\ref{fig:confound} (b) shows that verification and checking keywords increase both before and after \(k^\star(\tau)\). The increase before \(k^\star(\tau)\) is consistent with more useful checking before the verified answer appears. The increase after \(k^\star(\tau)\) suggests more unnecessary overthinking after the answer has already emerged. Figure~\ref{fig:confound} (c) presents a similar pattern within GRPO sampling groups: for the same prompt, successful rollouts repeat their extracted answers more often than unsuccessful rollouts.
A qualitative example in Appendix~\ref{fig:qualitative_rollouts} illustrates this pattern. 
At an earlier training step, the model still requires additional thinking to reach the correct answer; at a later step, it reaches the verified answer but continues with an additional verification step.

These observations motivate the analysis below: sequence-level credit can conflate the thinking that leads to answer emergence with subsequent overthinking, thereby reinforcing the latter when it is associated with positive advantage.

\subsection{Analysis of Overthinking under GRPO}

Under GRPO, each prompt is paired with multiple sampled trajectories. The complete output receives a reward, which is normalized within the sampled group to form a group-relative advantage \(A(\tau)\). The argument below is not the literal clipped GRPO estimator used in training. Instead, it is an unclipped first-order approximation for analysis, evaluated at or near \(\theta_{\mathrm{old}}\), that suppresses policy-ratio, clipping, and content-process terms in order to isolate where the scalar group advantage is broadcast. For readability, we describe the broadcast effect at the sentence level, while the training objective is applied at the token level.

Let \(\psi_s(\tau)=\nabla_\theta \log \pi_\theta(t_s\mid x,t_{<s})\) denote the score term for the \(s\)-th thinking sentence. Let \(Z_t(\tau)\) denote the number of unmasked thinking units in this simplified sentence-level view. Write \(\mathbb{E}_{\mathrm{old}}\) for expectation under \(\tau\sim\pi_{\theta_{\mathrm{old}}}\), and abbreviate \(k^\star(\tau)\) and \(K_t(\tau)\) as \(k^\star\) and \(K_t\) inside the display. For trajectories with no verified thinking prefix, our convention \(k^\star(\tau)=K_t(\tau)\) only makes the post-boundary sum empty; it should not be interpreted as answer emergence at the final sentence. A first-order thinking-component approximation can be written as
\begin{equation}
\label{eq:grad_split}
\begin{aligned}
g_{\mathrm{approx}}(\theta)
&= \mathbb{E}_{\mathrm{old}}\!\left[
\frac{A(\tau)}{Z_t(\tau)} C(\tau)
\right],\\[-0.1em]
C(\tau)
&=
\sum_{s\le k^\star}\psi_s(\tau)
+
\sum_{k^\star<s\le K_t}\psi_s(\tau).
\end{aligned}
\end{equation}

Equation~\eqref{eq:grad_split} shows the credit-broadcasting pathway: the same group-relative advantage multiplies score terms before and after the first verified answer boundary. This does not mean that every post-\(k^\star\) token is necessarily reinforced. Rather, when post-answer continuation is correlated with positive advantage, continuation after \(k^\star\) can inherit the same positive training signal as reasoning that may be needed to expose the verified answer. The matched-prompt evidence above supports this condition: successful rollouts contain more answer revisits and more post-answer checking behavior. Because \(Z_t(\tau)\) token-normalizes the rollout loss, this analysis concerns the direction of broadcast credit rather than gradient magnitude as post-\(k^\star\) continuation length increases.

This analysis suggests two intervention principles. First, the verified prefix before the editable boundary should be preserved, as it may contain reasoning needed to expose the answer. Second, the preference signal should target overthinking after a conservative boundary, rather than penalize total thinking length uniformly.

\begin{figure*}[t]
\centering
\includegraphics[width=\textwidth]{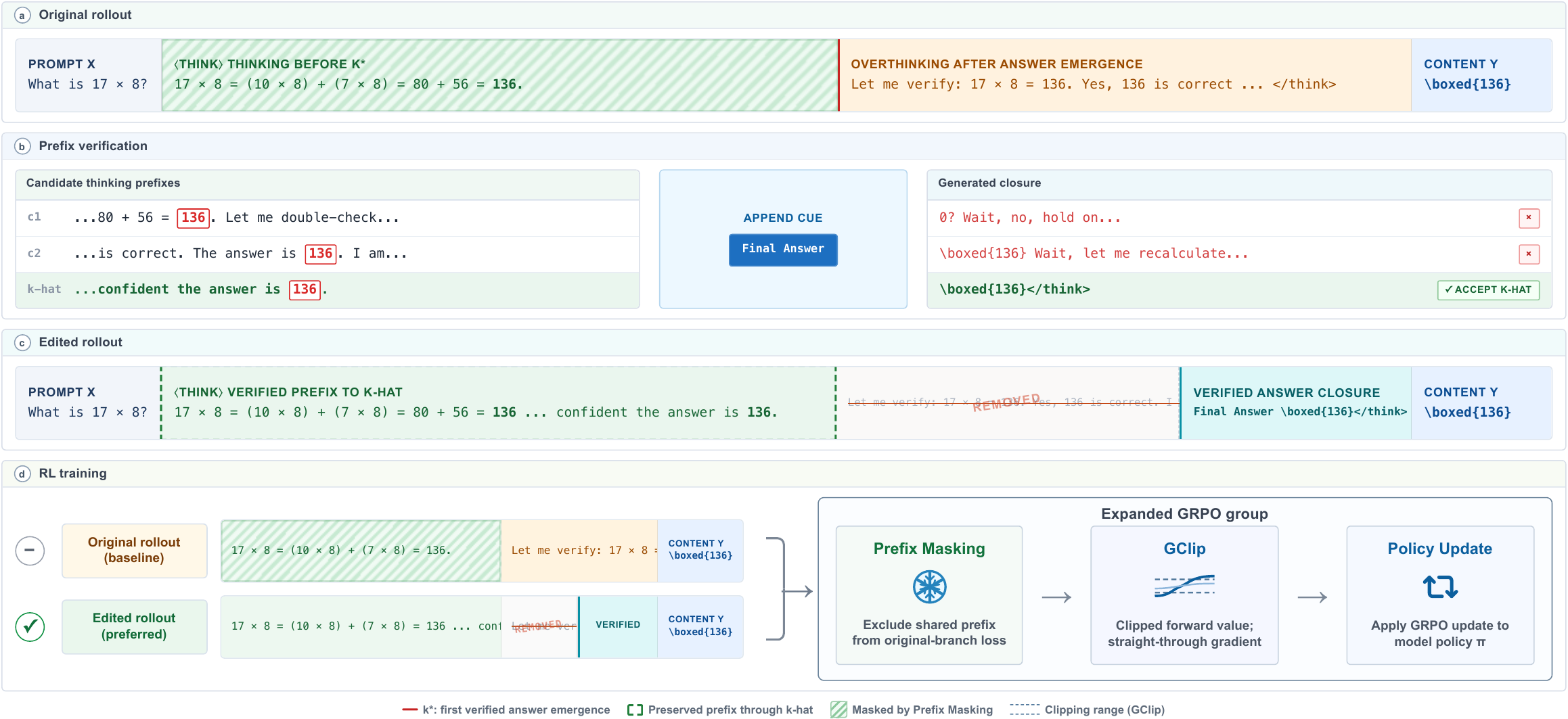}
\caption{Overview of dynamic rollout editing. After \(k^\star\) marks answer emergence, prefix verification selects an editable boundary \(\hat{k}\). The accepted edit preserves the thinking process through \(\hat{k}\), removes the subsequent overthinking, appends a verified final answer, and regenerates the content process. The expanded GRPO group assigns higher preference to accepted verified edits than to verified originals or unedited copies. Prefix masking prevents the shared verified prefix on the original branch from receiving direct negative credit, while GClip preserves learning signal for the preferred edited suffix.}

\label{fig:method}
\vspace{-12pt}
\end{figure*}

\section{Dynamic Rollout Editing}
\label{sec:method}
The analysis in Section~\ref{sec:pathology} motivates our intervention: when overthinking after a verified answer is correlated with positive advantage, trajectory-level rewards can reinforce it through the same gradient signal. We address this issue with \emph{dynamic rollout editing}, which constructs an edited counterpart for an overthinking rollout within the same GRPO group and assigns higher preference to the accepted edit than to the original. The edit preserves the verified prefix, ends the thinking process at an accepted editable boundary, and reduces preference for overthinking beyond that boundary.

\subsection{Edited Rollout Synthesis}

For each rollout, dynamic rollout editing searches near answer emergence for an editable boundary, proposes a trajectory that stops thinking there, and accepts it only if both prefix verification and regenerated-content verification succeed. Accepted edits preserve the verified prefix, remove later thinking, and regenerate the content process; failed edits fall back to an unedited copy.

\paragraph{Editable Boundary Identification.}
The answer-emergence boundary \(k^\star\) marks the earliest prefix from which the parser-verifier extracts the correct answer, but it is not necessarily a safe edit point: the answer may appear inside an unfinished calculation, provisional statement, or incomplete final-answer pattern. We therefore use \(k^\star\) as the starting point for prefix verification, not as the edit point itself. Starting from \(k^\star\), we append \verb|\nFinal Answer| to each candidate sentence prefix and let the policy generate a verified-answer closure. The first prefix whose closure exposes the same verified answer and reaches \texttt{</think>} is accepted as the editable boundary \(\hat{k}\).

\paragraph{Boundary-Based Rollout Synthesis.}
Given \(\hat{k}\), the edit preserves thinking through \(\hat{k}\), removes the original overthinking after \(\hat{k}\), appends the verified answer closure, and regenerates the content process. We keep the edited trajectory only when the regenerated content exposes the same verified answer. Otherwise, the method falls back to an unedited copy. Accepted edits therefore remove overthinking after \(\hat{k}\) without modifying the accepted verified prefix. Figure~\ref{fig:method} summarizes boundary identification, rollout synthesis, and training with edited rollouts.
Appendix~\ref{app:accepted_edit_case} shows a concrete accepted edit in this prefix-preserving form.

\subsection{Edited Rollout Scoring}

The scoring rule turns accepted edits into preferences within the group. For each sampled original trajectory \(\tau_i\), the GRPO group contains one auxiliary trajectory: the accepted edit \(\tilde{\tau}_i\) when both prefix verification and verification of the regenerated content succeed, and an unedited copy otherwise. Let \(V(\rho)=1\) indicate that the content process of trajectory \(\rho\) is verified, and let \(E(\rho)=1\) indicate that \(\rho\) is an accepted edit. The scoring rule is a shaped preference score for the expanded GRPO group, not a calibrated task reward. It encodes an additional preference among verified trajectories: when an accepted edit preserves the verified answer while removing continuation after \(\hat{k}\), the edit should be preferred over the original continuation.
\begin{equation}
\label{eq:trajectory_score}
q(\rho)=
\begin{cases}
1, & V(\rho)=1 \text{ and } E(\rho)=1,\\
0, & V(\rho)=1 \text{ and } E(\rho)=0,\\
-1, & V(\rho)=0.
\end{cases}
\end{equation}
Thus, accepted verified edits are ranked above verified originals or unedited copies, which are in turn ranked above unverified trajectories. The preference between an accepted edit and a verified original is not a claim that the edit is more correct; both pass the verifier. It is reward shaping that adds a training preference for stopping after a conservative verified boundary. Because the scores are normalized within the group to form advantages, their absolute scale is not calibrated; what matters is the induced within-group ranking and spacing.

When the original is also verified, the edit isolates the preference for removing post-boundary continuation while preserving the same verified answer. When the original is not verified but the edit is accepted, the edit is best interpreted as correcting answer drift after a verified prefix. This case changes both correctness and length, whereas verified-original edits isolate the preference for removing post-boundary continuation.

\subsection{Training with Edited Rollouts}
\label{sec:gclip}

This shaped preference changes the GRPO training signal without redefining task correctness: it weakens positive credit for the original thinking after the verified answer boundary.
For accepted edit pairs, this preference reversal creates two optimization issues during training, \emph{i.e.},

\textbf{(1) Negative credit on the shared prefix.}
The original trajectory receives a lower preference score even though it shares the verified prefix with the edited trajectory. Without a safeguard, the resulting lower advantage would be applied to the entire original trajectory, including both the accepted verified prefix and the overthinking after \(\hat{k}\).

\textbf{(2) Upper clipping on edited tokens.} The edited suffix begins with the verified final answer string and regenerated content rather than the original rollout continuation. These positively preferred tokens can have low probability under the old policy near \(\hat{k}\), so ordinary clipping can remove the ascent signal where the edit needs to be learned.

\paragraph{Prefix Masking for Preserving the Verified Prefix.}
For an accepted pair \((\tau,\tilde{\tau})\), the original and edited trajectories share the same generated prefix through \(\hat{k}\). After this prefix, the edited rollout begins the verified final answer string, whereas the original keeps thinking after \(\hat{k}\). Prefix masking applies only to the original trajectory: it excludes the generated tokens in this shared prefix from that trajectory's GRPO loss, while keeping unit weight on tokens after the divergence. The edited trajectory and all other trajectories keep unit weights. Thus, negative relative credit from the original targets its post-\(\hat{k}\) continuation rather than the accepted verified prefix shared with the edit.

\begin{table*}[t]
\centering
\newcommand{\tokpct}[2]{\makebox[3.0em][r]{#1}\,{\scriptsize\makebox[4.2em][r]{(#2\%)}}}
\resizebox{\textwidth}{!}{
\begin{tabular}{lrrrrrrrrrrrrr}
\toprule
\multirow{2}{*}{\textbf{Method}} & \multicolumn{2}{c}{\textbf{AIME24}} & \multicolumn{2}{c}{\textbf{AIME25}} & \multicolumn{2}{c}{\textbf{AIME26}} & \multicolumn{2}{c}{\textbf{GPQA-D}} & \multicolumn{2}{c}{\textbf{LCB-V6}} & \multicolumn{3}{c}{\textbf{Average}}\\
\cmidrule(lr){2-3} \cmidrule(lr){4-5} \cmidrule(lr){6-7} \cmidrule(lr){8-9} \cmidrule(lr){10-11} \cmidrule(lr){12-14}
 & \textbf{Think Tok.}$\downarrow$ & \textbf{Acc}$\uparrow$ & \textbf{Think Tok.}$\downarrow$ & \textbf{Acc}$\uparrow$ & \textbf{Think Tok.}$\downarrow$ & \textbf{Acc}$\uparrow$ & \textbf{Think Tok.}$\downarrow$ & \textbf{Acc}$\uparrow$ & \textbf{Think Tok.}$\downarrow$ & \textbf{Acc}$\uparrow$ & \textbf{Think Tok.}$\downarrow$ & \textbf{Acc}$\uparrow$ & \textbf{TE}$\uparrow$\\
\midrule
\multicolumn{14}{l}{\emph{Based on Qwen3-4B-Thinking-2507}}\\
\midrule
Raw Model & \tokpct{18923}{100.0} & \textbf{81.87} & \tokpct{21510}{100.0} & 79.79 & \tokpct{20789}{100.0} & 77.71 & \tokpct{8091}{100.0} & 65.47 & \tokpct{17722}{100.0} & \textbf{54.26} & \tokpct{17407}{100.0} & 71.82 & 4.13\\
GRPO & \tokpct{18322}{96.8} & 81.17 & \tokpct{20930}{97.3} & \textbf{80.83} & \tokpct{20284}{97.6} & 76.67 & \tokpct{7826}{96.7} & 65.26 & \tokpct{16333}{92.2} & 53.64 & \tokpct{16739}{96.2} & 71.51 & 4.27\\
GRPO + LP & \tokpct{19323}{102.1} & 79.50 & \tokpct{21059}{97.9} & 77.08 & \tokpct{20302}{97.7} & 75.62 & \tokpct{7906}{97.7} & \textbf{66.60} & \tokpct{16663}{94.0} & 53.57 & \tokpct{17051}{98.0} & 70.47 & 4.13\\
S-GRPO & \tokpct{19797}{104.6} & 80.00 & \tokpct{22298}{103.7} & 79.17 & \tokpct{21304}{102.5} & 78.33 & \tokpct{8435}{104.3} & 65.59 & \tokpct{18460}{104.2} & 54.14 & \tokpct{18059}{103.7} & 71.45 & 3.96\\
\midrule
\textbf{DRE} & \textbf{\tokpct{13092}{69.2}} & \textbf{81.87} & \textbf{\tokpct{14882}{69.2}} & 80.54 & \textbf{\tokpct{14431}{69.4}} & \textbf{79.37} & \textbf{\tokpct{5689}{70.3}} & 65.67 & \textbf{\tokpct{14630}{82.6}} & 53.79 & \textbf{\tokpct{12545}{72.1}} & \textbf{72.25} & \textbf{5.76}\\
\midrule
\multicolumn{14}{l}{\emph{Based on Qwen3-8B}}\\
\midrule
Raw Model & \tokpct{13988}{100.0} & 75.42 & \tokpct{17336}{100.0} & \textbf{68.75} & \tokpct{15862}{100.0} & 66.67 & \tokpct{8683}{100.0} & 59.34 & \tokpct{12177}{100.0} & 47.93 & \tokpct{13609}{100.0} & 63.62 & 4.67\\
GRPO & \tokpct{13773}{98.5} & 75.62 & \tokpct{16801}{96.9} & 64.79 & \tokpct{15581}{98.2} & 65.21 & \tokpct{8474}{97.6} & 60.61 & \tokpct{12197}{100.2} & \textbf{49.50} & \tokpct{13365}{98.2} & 63.15 & 4.73\\
GRPO + LP & \tokpct{12968}{92.7} & 74.17 & \tokpct{16919}{97.6} & 66.04 & \tokpct{15107}{95.2} & 63.96 & \tokpct{8159}{94.0} & 60.10 & \tokpct{11995}{98.5} & 48.64 & \tokpct{13030}{95.7} & 62.58 & 4.80\\
S-GRPO & \tokpct{16133}{115.3} & 74.17 & \tokpct{18090}{104.3} & 65.83 & \tokpct{17707}{111.6} & 62.50 & \tokpct{9660}{111.3} & 59.79 & \tokpct{12293}{101.0} & 48.86 & \tokpct{14777}{108.6} & 62.23 & 4.21\\
\midrule
\textbf{DRE} & \textbf{\tokpct{10163}{72.7}} & \textbf{75.83} & \textbf{\tokpct{13448}{77.6}} & 66.67 & \textbf{\tokpct{11381}{71.8}} & \textbf{66.87} & \textbf{\tokpct{6658}{76.7}} & \textbf{60.98} & \textbf{\tokpct{11210}{92.1}} & 48.93 & \textbf{\tokpct{10572}{77.7}} & \textbf{63.86} & \textbf{6.04}\\
\midrule
\multicolumn{14}{l}{\emph{Based on Qwen3-30B-A3B-Thinking-2507}}\\
\midrule
Raw Model & \tokpct{15299}{100.0} & 88.12 & \tokpct{18095}{100.0} & \textbf{85.62} & \tokpct{16985}{100.0} & \textbf{88.96} & \tokpct{6922}{100.0} & \textbf{76.45} & \tokpct{15559}{100.0} & 64.64 & \tokpct{14572}{100.0} & \textbf{80.76} & 5.54\\
GRPO & \tokpct{13784}{90.1} & 88.75 & \tokpct{15600}{86.2} & 81.87 & \tokpct{15158}{89.2} & 84.58 & \tokpct{6236}{90.1} & 74.68 & \tokpct{15413}{99.1} & 63.43 & \tokpct{13238}{90.8} & 78.66 & 5.94\\
GRPO + LP & \tokpct{14374}{94.0} & \textbf{88.96} & \tokpct{16206}{89.6} & 82.71 & \tokpct{16039}{94.4} & 85.00 & \tokpct{6435}{93.0} & 75.00 & \tokpct{15494}{99.6} & 63.93 & \tokpct{13710}{94.1} & 79.12 & 5.77\\
S-GRPO & \tokpct{15821}{103.4} & 88.33 & \tokpct{20501}{113.3} & 81.67 & \tokpct{19133}{112.6} & 85.00 & \tokpct{7000}{101.1} & 74.68 & \tokpct{17028}{109.4} & 64.79 & \tokpct{15897}{109.1} & 78.89 & 4.96\\
\midrule
\textbf{DRE} & \textbf{\tokpct{10407}{68.0}} & 88.54 & \textbf{\tokpct{13040}{72.1}} & 83.54 & \textbf{\tokpct{12153}{71.6}} & 86.46 & \textbf{\tokpct{5318}{76.8}} & 75.51 & \textbf{\tokpct{13977}{89.8}} & \textbf{64.86} & \textbf{\tokpct{10979}{75.3}} & 79.78 & \textbf{7.27}\\
\bottomrule
\end{tabular}
}
\caption{Main comparison across three Qwen3 reasoning models. Percentages in parentheses report thinking-token use relative to the Raw Model under the same base model and benchmark. The best result is shown in \textbf{bold}.}
\label{tab:method_comparison}
\end{table*}

\paragraph{GClip for Learning from Edited Rollouts.}
\emph{GClip} handles the second issue by using a clipped ratio operator with a custom backward pass for edited tokens. Accepted edits are synthetic continuations constructed after rollout sampling, so the expanded-group objective should be read as a GRPO-style preference surrogate over original and edited trajectories rather than as an unbiased on-policy GRPO estimator. We retain the ratio to \(\pi_{\theta_{\mathrm{old}}}\) to keep the forward surrogate clipped and to measure compatibility with the old policy. Because the edited suffix no longer follows the original rollout path, its early tokens can have very low old-policy probability near \(\hat{k}\). Positively preferred edited tokens can then hit the upper clip boundary quickly, causing ordinary clipping to remove the ascent signal where the edit needs to be learned.

For generated tokens \(z_1,\ldots,z_m\), let the token-level policy ratio be
\begin{equation}
\label{eq:policy_ratio}
r_s(\theta)\triangleq\frac{\pi_\theta(z_s \mid x,z_{<s})}{\pi_{\theta_{\mathrm{old}}}(z_s \mid x,z_{<s})}.
\end{equation}
For any scalar ratio \(r>0\) and clip bounds \(0<\ell<u\), GClip has the same forward scalar value as the ordinary clipped ratio, but it changes the backward pass through a stop-gradient construction:
\begin{equation}
\label{eq:gclip}
\operatorname{GClip}(r;\ell,u)\triangleq
\operatorname{sg}\!\bigl(\mathrm{clip}(r,\ell,u)\bigr)
\cdot \frac{r}{\operatorname{sg}(r)}.
\end{equation}
Here \(\operatorname{sg}(\cdot)\) denotes the stop-gradient operator. GClip should therefore be understood as a straight-through clipped-ratio operator: it preserves the clipped forward surrogate value while replacing the zero-gradient behavior of the ordinary clipped branch outside the clipping interval. In the training objective below, GClip replaces the clipped-ratio term in the GRPO-style surrogate. Theorem~\ref{thm:first_order_consistency} formalizes this straight-through derivative.

\paragraph{Overall Training Objective.}
Let \(A(\rho)\) be the group-relative advantage obtained by normalizing the scores in Equation~\eqref{eq:trajectory_score} within the expanded GRPO group, and write \(\operatorname{GClip}_\epsilon(r)=\operatorname{GClip}(r;1-\epsilon,1+\epsilon)\). We integrate GClip into GRPO by replacing the clipped-ratio term with \(\operatorname{GClip}_\epsilon\), giving the token-level surrogate
\begin{equation}
\label{eq:token_loss}
\begin{aligned}
\ell_s(\theta;\rho)
&=-\min\!\Bigl(
r_s(\theta)A(\rho),\\
&\qquad\qquad
\operatorname{GClip}_\epsilon(r_s(\theta))A(\rho)
\Bigr).
\end{aligned}
\end{equation}
For accepted edited suffixes, the main saturation case is \(A(\rho)>0\) with \(r_s>1+\epsilon\): ordinary clipping would remove the ascent signal on the clipped branch. GClip keeps the same clipped forward value while supplying a straight-through backward signal, so preferred edited tokens can still be learned. This should be read as an optimization surrogate rather than a formal PPO trust-region guarantee; Appendix~\ref{app:gclip_theorem} gives the full derivative analysis, including the lower-side saturation case.

For an original trajectory paired with an accepted edit, \(w_s(\rho)=0\) on generated tokens in the shared prefix through \(\hat{k}\) and \(w_s(\rho)=1\) otherwise. For its edited counterpart and all other trajectories, \(w_s(\rho)=1\) at every generated token. The rollout loss is
\begin{equation}
\label{eq:rollout_loss}
\mathcal{L}(\theta;\rho)
=
\frac{\sum_s w_s(\rho)\ell_s(\theta;\rho)}
{\sum_s w_s(\rho)}.
\end{equation}
This objective targets post-\(\hat{k}\) thinking continuation in accepted edited pairs while preventing the lower-scored original branch from directly penalizing the shared verified prefix.

\section{Experiments}
\label{sec:experiments}

\subsection{Experimental Setups}

\vpara{Benchmarks.}
We conduct experiments on diverse benchmarks spanning mathematical reasoning, science QA, and code generation: AIME24, AIME25, and AIME26~\cite{maa_aime} for math, GPQA Diamond (GPQA-D)~\cite{rein2023gpqagraduatelevelgoogleproofqa} for graduate-level science QA, and LiveCodeBench V6 (LCB-V6)~\cite{DBLP:conf/iclr/JainHGLYZWSSS25} for code.

\vpara{Metrics.}
We report accuracy (Acc), mean thinking tokens (Think Tok.), and token efficiency (TE). \textbf{Acc} indicates task performance, and \textbf{Think Tok.} measures the average length of the generated thinking process. We also report token efficiency (\textbf{TE}) to summarize the accuracy--token tradeoff, defined as \(\mathrm{TE}=1000\cdot\overline{\mathrm{Acc}}/\overline{L_t}\), where \(\overline{\mathrm{Acc}}\) and \(\overline{L_t}\) denote the average accuracy and thinking-token count over the benchmarks in the table. More details are provided in Appendix~\ref{app:experimental_details}.

\vpara{Baselines.} We evaluate Dynamic Rollout Editing (DRE) on three Qwen3 reasoning models: Qwen3-4B-Thinking-2507, Qwen3-8B, and Qwen3-30B-A3B-Thinking-2507~\cite{yang2025qwen3technicalreport}. For each backbone, \emph{Raw Model} denotes the corresponding original Qwen3 model without further training. We compare DRE with Raw Model, GRPO~\cite{DBLP:journals/corr/abs-2402-03300}, GRPO + LP~\cite{DBLP:journals/corr/abs-2501-12599}, and S-GRPO~\cite{DBLP:journals/corr/abs-2505-07686}. All trainable methods are trained on DAPO-Math-17k~\cite{DBLP:journals/corr/abs-2503-14476}.

\begin{figure*}[t]
\centering
\includegraphics[width=\textwidth]{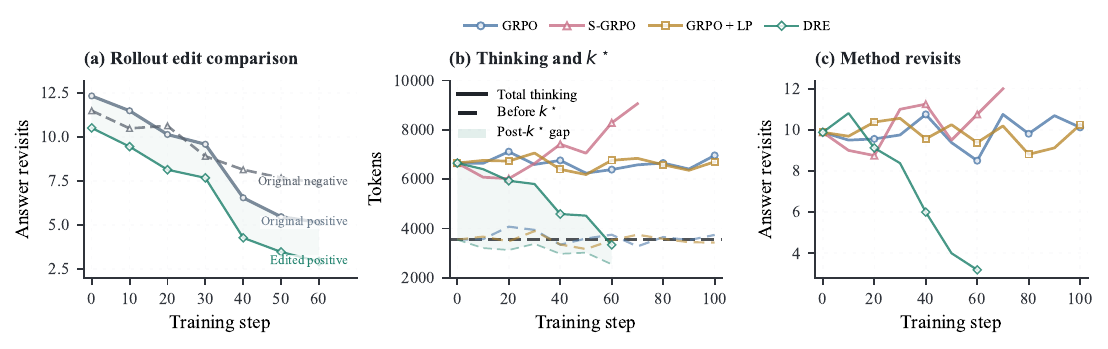}
\caption{Evaluation on selected AIME24 samples across training steps. Colors and markers indicate different methods. In panel (b), solid curves show total thinking length, dashed curves show thinking up to \(k^\star\), and the shaded region highlights the reduction after \(k^\star\) under DRE. \textbf{(a)} Edited positive trajectories contain fewer answer revisits than their original counterparts. \textbf{(b)} DRE reduces thinking mainly after \(k^\star\). \textbf{(c)} Answer revisits decrease more substantially under DRE than under the baselines.}
\label{fig:aime24_method_dynamics}
\end{figure*}

\subsection{Main Results}
As shown in Table~\ref{tab:method_comparison}, we make four empirical observations: \textbf{(1) DRE substantially reduces thinking tokens while preserving benchmark accuracy.} Among methods with complete LCB-V6 results, DRE has the lowest five-benchmark average thinking length and the highest token efficiency in every model block, using only 72.1\%, 77.7\%, and 75.3\% of Raw Model thinking tokens on the three backbones. On Qwen3-4B-Thinking-2507 and Qwen3-8B, this reduction is accompanied by higher average accuracy than both Raw Model and GRPO. On Qwen3-30B-A3B-Thinking-2507, where all trained variants trail Raw Model, DRE stays closest to Raw Model in accuracy while using the fewest thinking tokens among trained methods. The largest-model result should therefore be read as stronger length control with competitive accuracy, rather than as a uniform accuracy gain.
\textbf{(2) The reduction is not confined to the math training distribution.} Although all trainable methods use DAPO-Math-17k, DRE shortens thinking across held-out AIME evaluations, GPQA-D, and out-of-domain LCB-V6. Relative to GRPO, DRE uses fewer thinking tokens on all nine AIME model--benchmark rows, all three GPQA-D rows, and all three LCB-V6 rows. Its accuracy also matches or exceeds GRPO on seven AIME rows, improves GPQA-D accuracy in each model block, and stays within 0.6 accuracy points of the best completed method on LCB-V6. This cross-benchmark pattern is consistent with reducing redundant post-solution continuation rather than only pruning math-specific thinking.
\textbf{(3) Generic length pressure does not reproduce the same pattern.} GRPO + LP sometimes shortens thinking relative to GRPO, but it remains longer than DRE on the five-benchmark average and yields less stable accuracy. This contrast is consistent with the design difference: DRE edits after the verified solution boundary while preserving the solution-reaching prefix, whereas a length penalty does not distinguish necessary thinking from redundant continuation.
\textbf{(4) The edited preference signal covers most sampled trajectories.} During DRE training, the edit-and-verify procedure accepted edits for 87.20\% of sampled trajectories, and the remaining 12.80\% used the conservative unedited fallback. Among accepted edits, 99.40\% preserved verifier correctness after content regeneration. Thus the preference term in Equation~\eqref{eq:trajectory_score} is supported by most sampled rollouts, rather than by a small exceptional subset.

\subsection{Overthinking Analysis}

We next examine whether the reduction in thinking length occurs mainly after answer emergence. Figure~\ref{fig:aime24_method_dynamics} evaluates selected AIME24 samples across training steps. Panel (a) compares original positive trajectories, original negative trajectories, and edited positive trajectories. Original positive trajectories show more answer revisits than original negative trajectories, whereas edited positive trajectories contain fewer revisits than their original counterparts and become the group with the fewest revisits during training. Panel (b) shows that DRE reduces total thinking length much more than thinking before \(k^\star\), indicating that most of the reduction occurs after answer emergence. Panel (c) shows the same trend at the method level: answer revisits decrease more under DRE than under the baselines. In contrast, GRPO + LP leaves these measures largely unchanged, suggesting that the effect is not explained by a generic preference for shorter trajectories. Appendix~\ref{app:semantic_dre_dynamics} provides a complementary semantic-space visualization of the same intervention effect, and Appendix~\ref{app:ablations} reports ablations that remove each component individually.

\section{Related Work}
Prior work reduces overthinking by controlling reasoning length during post-training, shortening inference-time generation, or explicitly deciding when to stop~\cite{DBLP:journals/corr/abs-2503-16419,DBLP:journals/corr/abs-2503-24377}. Post-training methods often use short or variable-length thinking supervision, length rewards or penalties, or latent reasoning schemes~\cite{DBLP:journals/corr/abs-2407-06023,DBLP:journals/corr/abs-2501-12599,DBLP:journals/corr/abs-2503-04697,DBLP:journals/corr/abs-2412-06769}. Inference-time methods instead shorten an already trained policy through prompts, decoding controls, early-exit signals, or stopping controllers~\cite{DBLP:journals/corr/abs-2504-15895,DBLP:journals/corr/abs-2505-07686,wei2025evolutionthought}. Our work is complementary: we analyze how GRPO-style credit assignment can reinforce overthinking after answer emergence, and use dynamic rollout editing to preserve necessary reasoning while reducing preference for unnecessary thinking after the answer becomes verifiable. Appendix~\ref{app:related_work} gives a fuller discussion.
\section{Conclusion}

We study why RL-trained reasoning models continue thinking after a correct answer has already emerged. Our analysis suggests that, in GRPO-style training, successful rollouts can contain more overthinking than unsuccessful rollouts sampled for the same prompts, and sequence-level rewards may assign positive credit to both reaching the answer and overthinking after it. In this work, we propose DRE, which prefers verified shorter trajectories within the same RL group while preserving the reasoning prefix that reaches the answer. 
Experiments on diverse benchmarks covering mathematical reasoning, science QA, and code generation validate the effectiveness of DRE. The results show that DRE reduces thinking length while preserving overall task performance, and that this reduction occurs mainly after answer emergence.
\section*{Limitations}

RL training in this work is conducted only on DAPO-Math-17K~\cite{DBLP:journals/corr/abs-2503-14476}, a dataset of math problems with verifiable answers. This setting supports reliable identification of answer emergence, which is required by dynamic rollout editing, but it also limits the diversity of training signals used to learn shorter reasoning behavior. Although the resulting model reduces thinking tokens on out-of-domain code-generation benchmarks, the gains are smaller than those observed on math tasks. Extending dynamic rollout editing to more diverse verifiable RL training data remains an important direction for improving and assessing token savings across domains.
\section*{Ethical Considerations}

This work aims to reduce overthinking after answer emergence in RL-trained reasoning models. The main potential benefit is improved inference efficiency, which can lower compute cost and associated resource use. Since shorter thinking processes may not be equally suitable for all tasks, deployments should be preceded by evaluation in the intended target domain, especially when additional verification steps are important for reliability.

\bibliography{references}

\appendix

\section{Related Work}
\label{app:related_work}

We organize related work according to the stage at which prior methods intervene to control reasoning length: during post-training, at inference through prompting or decoding, or via explicit early-exit mechanisms. This organization follows recent taxonomies of efficient reasoning~\cite{DBLP:journals/corr/abs-2503-16419,DBLP:journals/corr/abs-2503-24377}.

\paragraph{Post-training methods.}
Training-time approaches control reasoning length by modifying supervision signals, reward functions, or the latent reasoning space. Representative directions include supervised fine-tuning with variable-length chain-of-thought data~\cite{DBLP:journals/corr/abs-2407-06023,DBLP:conf/aaai/KangSCZ25,DBLP:journals/corr/abs-2502-12067,DBLP:journals/corr/abs-2502-09601}, RL objectives with explicit length rewards or penalties~\cite{DBLP:journals/corr/abs-2501-12599,DBLP:journals/corr/abs-2501-12570,DBLP:journals/corr/abs-2503-04697}, and latent-space reasoning schemes~\cite{DBLP:journals/corr/abs-2412-06769,DBLP:journals/corr/abs-2412-13171,DBLP:journals/corr/abs-2502-21074}. The proposed method also intervenes during post-training, but differs in mechanism: rather than globally rewarding concise reasoning or requiring short chain-of-thought supervision, it edits high-reward RL rollouts at the point where a verified solution first appears.

\paragraph{Prompt-based methods.}
Prompting strategies elicit shorter or more disciplined chains of thought by instructing the model to be concise, decisive, or self-monitoring~\cite{DBLP:journals/corr/abs-2412-18547,DBLP:journals/corr/abs-2501-19393,DBLP:journals/corr/abs-2503-01141,openai2025gptoss120bgptoss20bmodel,DBLP:conf/aaai/KangSCZ25,DBLP:journals/corr/abs-2505-23480,DBLP:journals/corr/abs-2508-05337}. These methods are appealing because they require no retraining, but they tend to be brittle: their effectiveness depends heavily on prompt wording, and a single prompting policy often generalizes poorly across tasks with different reasoning depths or answer formats.

\paragraph{Early-exit methods.}
Early-exit approaches intervene directly in the generation process by terminating reasoning once the model appears to have reached a solution. Some methods rely on external monitors or additional training, such as auxiliary classifiers~\cite{liu-wang-2025-answer}, bandit controllers~\cite{DBLP:journals/corr/abs-2510-10103}, difficulty estimators~\cite{DBLP:journals/corr/abs-2504-13367}, RL-based stopping policies such as S-GRPO~\cite{DBLP:journals/corr/abs-2505-07686}, or self-aware guided sampling as in SAGE-RL~\cite{DBLP:journals/corr/abs-2602-08354}. Others pursue lower-overhead stopping rules based on intrinsic model signals, such as DEER~\cite{DBLP:journals/corr/abs-2504-15895}, or are motivated by reasoning-dynamics analyses that characterize an instance-specific post-convergence regime via a reasoning completion point~\cite{wei2025evolutionthought}. The present work is complementary to this line of research: rather than stopping generation through a controller or sampling rule, it edits sampled rollouts after verified answer emergence and trains against the later continuation.

\section{Experimental Details}
\label{app:experimental_details}

\paragraph{Training data and rollout sampling.}
All training-time variants are trained on DAPO-Math-17k~\cite{DBLP:journals/corr/abs-2503-14476}. During training, we use a maximum generation length of 28,000 tokens. Each optimization step uses a batch size of 64 prompts, and each prompt contributes an effective group of 8 trajectories or candidates for group-relative normalization. For GRPO and GRPO + LP, these are 8 sampled rollouts. For DRE, we sample 4 original rollouts and construct 4 auxiliary counterparts; each auxiliary trajectory is an accepted edit when both prefix verification and verification of the regenerated content process succeed, and an unedited copy otherwise. For S-GRPO, we sample 1 rollout and construct 8 candidates from that rollout following its stopping-policy construction.

\paragraph{Training horizon and checkpoint selection.}
For DRE, we select the first checkpoint whose validation-set mean full thinking length is shorter than the raw model's validation-set mean \(k^\star\) token position. This stopping rule selects the point at which the trained model's complete thinking process has become shorter than the raw model's answer-emergence boundary on the validation set. For the other trained methods, we train to 100 optimization steps and report the step-100 checkpoint. In S-GRPO runs that collapse before step 100 on a training set, we instead report the last checkpoint before collapse and do not use post-collapse checkpoints.

\paragraph{Optimization hyperparameters.}
We do not use a KL penalty in any training variant, i.e., the KL coefficient is set to zero. For ordinary clipped GRPO objectives, including the GRPO, GRPO + LP, S-GRPO, and \emph{Standard Clip} settings, we use asymmetric ratio clipping with lower-side \(\epsilon_{\downarrow}=0.20\) and upper-side \(\epsilon_{\uparrow}=0.28\), corresponding to ratio bounds \([0.80,1.28]\). For GClip in DRE and the GClip-based ablations, we use \(\epsilon_{\downarrow}=\epsilon_{\uparrow}=0.10\), corresponding to ratio bounds \([0.90,1.10]\).

\paragraph{GRPO + LP length reward.}
GRPO + LP is a length-penalty baseline following the group-relative length reward used in Kimi k1.5~\cite{DBLP:journals/corr/abs-2501-12599}. Within each sampled group, let \(r(x,y_i,y^\star)\in\{0,1\}\) be the rule-based correctness indicator for response \(y_i\), let \(\operatorname{len}(i)\) be its generated length, and let \(\operatorname{min\_len}\) and \(\operatorname{max\_len}\) be the minimum and maximum generated lengths in the same group. If \(\operatorname{max\_len}=\operatorname{min\_len}\), we set \(\lambda_i=0\). Otherwise, we set
\[
\lambda_i
=0.5-
\frac{\operatorname{len}(i)-\operatorname{min\_len}}
{\operatorname{max\_len}-\operatorname{min\_len}},
\]
and define the auxiliary length reward as
\[
r_{\mathrm{len}}(i)=
\begin{cases}
\lambda_i, & r(x,y_i,y^\star)=1,\\
\min(0,\lambda_i), & r(x,y_i,y^\star)=0.
\end{cases}
\]
Thus shorter correct responses receive larger auxiliary reward, while incorrect responses receive no positive length bonus.

\paragraph{Evaluation protocol.}
At evaluation time, we use temperature 0.6 for all benchmarks. Unless otherwise noted, for AIME (AIME I and II combined for each year), we sample 16 responses per problem with a maximum generation length of 81,920 tokens. For GPQA-D, we sample 8 responses per problem with a maximum generation length of 32,768 tokens. For LCB-V6, we sample 8 responses per problem with a maximum generation length of 81,920 tokens. For Qwen3-8B, we use a 40,960-token generation cap for all benchmarks due to our serving and evaluation configuration. A small number of method--benchmark runs use fewer completed evaluation samples because of time constraints. We report Acc as per-sample accuracy over completed sampled responses. Think Tok. is the mean number of generated tokens inside the delimited thinking process, counted between the first \texttt{<think>} marker and the first subsequent \texttt{</think>} marker. If \texttt{</think>} is absent, we count generated tokens after \texttt{<think>} until the generation stops or reaches the maximum generation length. Truncated outputs are scored as generated, without repair or regeneration. For mechanism-facing analyses, we additionally track mean \(k^\star\), answer revisits after \(k^\star\), and old-policy token probabilities on accepted edited suffixes after \(\hat{k}\).

\paragraph{Token efficiency.}
In Table~\ref{tab:method_comparison}, \textbf{TE} denotes token efficiency:
\[
\mathrm{TE}=1000\cdot\frac{\overline{\mathrm{Acc}}}{\overline{L_t}},
\]
where \(\overline{\mathrm{Acc}}\) is the average accuracy in percentage points and \(\overline{L_t}\) is the average thinking-token count over the benchmarks included in the table. TE can be read as accuracy points per 1,000 thinking tokens, so larger values indicate a better accuracy--token tradeoff.

\section{Ablations}
\label{app:ablations}

We report leave-one-out ablations evaluated on selected AIME24 samples across training steps. The \emph{Full Method} combines prefix verification, boundary-based editing, prefix masking, and GClip. \emph{No Prefix Verification} removes the prefix-closure check while retaining the final regenerated-content correctness check. \emph{No Prefix Masking} removes the pair-local mask that excludes the original branch's direct loss on the shared thinking prefix. \emph{Standard Clip} replaces GClip with ordinary clipping near the edited boundary and in the early edited suffix. These ablations test whether each safeguard is needed to reduce overthinking while preserving task performance. Figure~\ref{fig:aime24_ablation_dynamics} reports accuracy, total thinking length, thinking before \(k^\star\), and content process length.

\begin{figure*}[t]
\centering
\includegraphics[width=0.96\textwidth]{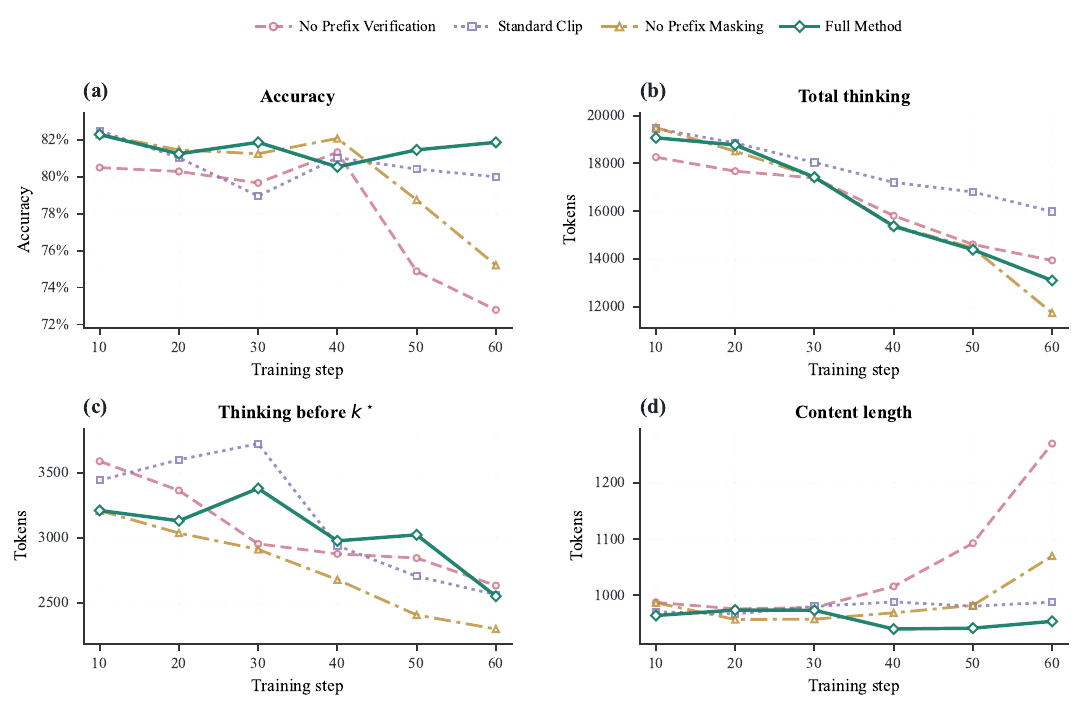}
\caption{Leave-one-out ablations evaluated on selected AIME24 samples across training steps. The \emph{Full Method} maintains accuracy while reducing total thinking and keeping content process length stable. \textbf{(a)} Accuracy. \textbf{(b)} Total thinking. \textbf{(c)} Thinking before \(k^\star\). \textbf{(d)} Content process length.}
\label{fig:aime24_ablation_dynamics}
\end{figure*}

The ablations show distinct failure modes. Without prefix verification, accuracy declines while content length increases, suggesting that some accepted prefixes are not yet ready to expose the verified solution and the regenerated content process must complete reasoning that should have remained in thinking. Without prefix masking, thinking before \(k^\star\) drops the most and accuracy later degrades, consistent with negative credit leaking into the shared verified prefix. With \emph{Standard Clip}, total thinking declines more slowly than under the \emph{Full Method}, consistent with ordinary clipping weakening the intended ascent signal near the edited boundary; Appendix~\ref{app:gclip_mismatch} gives a representative old-policy mismatch example. Overall, the ablation results support combining prefix verification, prefix masking, and GClip.

\section{Supplementary Early-Exit Comparison}
\label{app:early_exit_comparison}

Table~\ref{tab:deer_dre_appendix} situates DRE against two training-free inference-time stopping baselines, DEER~\cite{DBLP:journals/corr/abs-2504-15895} and RCPD~\cite{wei2025evolutionthought}. We report this supplementary comparison on AIME24, AIME25, AIME26, and GPQA-D, where we ran both early-exit baselines under the same evaluation pipeline. LCB-V6 is omitted from this table because we did not validate RCPD under the same LCB-V6 protocol. The Average columns are therefore computed over these four benchmarks and should be interpreted within this supplementary comparison, rather than compared directly with the five-benchmark averages in Table~\ref{tab:method_comparison}.

\begin{table*}[t]
\centering
\resizebox{\textwidth}{!}{
\begin{tabular}{llrrrrrrrrrr}
\toprule
\multirow{2}{*}{\textbf{Backbone}} & \multirow{2}{*}{\textbf{Method}} & \multicolumn{2}{c}{\textbf{AIME24}} & \multicolumn{2}{c}{\textbf{AIME25}} & \multicolumn{2}{c}{\textbf{AIME26}} & \multicolumn{2}{c}{\textbf{GPQA-D}} & \multicolumn{2}{c}{\textbf{Average}}\\
\cmidrule(lr){3-4} \cmidrule(lr){5-6} \cmidrule(lr){7-8} \cmidrule(lr){9-10} \cmidrule(lr){11-12}
 & & \textbf{Think Tok.}$\downarrow$ & \textbf{Acc}$\uparrow$ & \textbf{Think Tok.}$\downarrow$ & \textbf{Acc}$\uparrow$ & \textbf{Think Tok.}$\downarrow$ & \textbf{Acc}$\uparrow$ & \textbf{Think Tok.}$\downarrow$ & \textbf{Acc}$\uparrow$ & \textbf{Think Tok.}$\downarrow$ & \textbf{Acc}$\uparrow$\\
\midrule
Qwen3-4B-Thinking-2507 & Raw Model & 18923 & 81.87 & 21510 & 79.79 & 20789 & 77.71 & 8091 & 65.47 & 17328 & 76.21\\
Qwen3-4B-Thinking-2507 & DEER & 18620 & 80.74 & 20651 & 77.41 & 19219 & 78.52 & 7914 & 66.72 & 16601 & 75.85\\
Qwen3-4B-Thinking-2507 & RCPD & 14384 & 79.37 & 16049 & 78.33 & 16916 & 76.04 & 7265 & 66.04 & 13653 & 74.95\\
Qwen3-4B-Thinking-2507 & DRE & 13092 & 81.87 & 14882 & 80.54 & 14431 & 79.37 & 5689 & 65.67 & 12024 & 76.86\\
\midrule
Qwen3-8B & Raw Model & 13988 & 75.42 & 17336 & 68.75 & 15862 & 66.67 & 8683 & 59.34 & 13967 & 67.55\\
Qwen3-8B & DEER & 12219 & 72.08 & 14695 & 64.17 & 14324 & 59.26 & 8138 & 60.98 & 12344 & 64.12\\
Qwen3-8B & RCPD & 12952 & 74.79 & 15332 & 64.79 & 14857 & 62.50 & 7330 & 59.85 & 12618 & 65.48\\
Qwen3-8B & DRE & 10163 & 75.83 & 13448 & 66.67 & 11381 & 66.87 & 6658 & 60.98 & 10413 & 67.59\\
\midrule
Qwen3-30B-A3B-Thinking-2507 & Raw Model & 15299 & 88.12 & 18095 & 85.62 & 16985 & 88.96 & 6922 & 76.45 & 14325 & 84.79\\
Qwen3-30B-A3B-Thinking-2507 & DEER & 14405 & 87.78 & 18090 & 84.08 & 17707 & 86.29 & 6528 & 75.47 & 14183 & 83.40\\
Qwen3-30B-A3B-Thinking-2507 & RCPD & 13831 & 87.08 & 15231 & 83.54 & 15741 & 85.83 & 6021 & 74.94 & 12706 & 82.85\\
Qwen3-30B-A3B-Thinking-2507 & DRE & 10407 & 88.54 & 13040 & 83.54 & 12153 & 86.46 & 5318 & 75.51 & 10230 & 83.51\\
\bottomrule
\end{tabular}
}
\caption{Supplementary comparison with Raw Model, training-free early-exit methods, and DRE on the shared evaluation set. Average is computed over AIME24, AIME25, AIME26, and GPQA-D.}
\label{tab:deer_dre_appendix}
\end{table*}

\FloatBarrier

\nolinenumbers

\section{GClip as a Straight-Through Clipped-Ratio Operator}
\label{app:gclip_theorem}

This appendix formalizes the straight-through property of GClip used in Section~\ref{sec:gclip}. Throughout, let \(0<\ell<u\) and define
\[
\operatorname{GClip}(x;\ell,u)\;\triangleq\;
\operatorname{sg}\!\bigl(\mathrm{clip}(x,\ell,u)\bigr)\cdot\frac{x}{\operatorname{sg}(x)},
\]
for any policy ratio \(x>0\).
Here \(\operatorname{sg}\) denotes stop-gradient: it leaves the forward value unchanged but blocks gradients through its argument.

\begin{theorem}[GClip identity and straight-through derivative]
\label{thm:first_order_consistency}
Let \(s(x)=\operatorname{GClip}(x;\ell,u)\). In the forward pass,
\[
s(x)=\mathrm{clip}(x,\ell,u).
\]
Under automatic differentiation with stop-gradient, the scalar derivative with respect to \(x\) is
\[
D^{\mathrm{sg}}s(x)
=\frac{\mathrm{clip}(x,\ell,u)}{x}
=\mathrm{clip}\!\left(1,\frac{\ell}{x},\frac{u}{x}\right),
\]
or equivalently,
\[
D^{\mathrm{sg}}s(x)=
\begin{cases}
\ell/x, & 0<x<\ell,\\
1, & \ell\le x\le u,\\
u/x, & x>u.
\end{cases}
\]

Consequently, let \(x=x_\theta\) be a policy ratio that depends on the model parameters, and let \(L(\theta)=-A\,s(x_\theta)\) denote the clipped branch of a token-level loss for a positively preferred token, where \(A>0\) is treated as constant during differentiation. If \(x_\theta>u\), then
\[
\nabla_\theta^{\mathrm{sg}} L(\theta)
=-A\,\frac{u}{x_\theta}\,\nabla_\theta x_\theta,
\]
which is non-zero whenever \(\nabla_\theta x_\theta\neq 0\). By contrast, the standard clipped branch
\[
L_{\mathrm{clip}}(\theta)=-A\,\mathrm{clip}(x_\theta,\ell,u)
\]
has zero automatic-differentiation gradient for \(x_\theta>u\). Thus GClip leaves the clipped forward scalar value unchanged while replacing the ordinary clipped branch with a nonzero straight-through gradient in this selected saturation case.

Moreover, for every interior point \(x_0\in(\ell,u)\),
\[
s(x_0)=x_0
\quad\text{and}\quad
D^{\mathrm{sg}}s(x_0)=1.
\]
Therefore, inside the clipping interval, this token-level branch has the same first-order direction as the unclipped branch.
\end{theorem}

\begin{proof}
\noindent\textit{Forward value.}
Fix any \(x>0\). Since \(\operatorname{sg}\) leaves the forward value unchanged, the first factor equals \(\mathrm{clip}(x,\ell,u)\) in the forward pass. The second factor satisfies
\[
\frac{x}{\operatorname{sg}(x)}=1
\]
in the forward pass because the denominator has the same numerical value as \(x\). Therefore
\[
s(x)=\operatorname{sg}\!\bigl(\mathrm{clip}(x,\ell,u)\bigr)\cdot\frac{x}{\operatorname{sg}(x)}
=\mathrm{clip}(x,\ell,u),
\]
which proves the first claim.

\smallskip
\noindent\textit{Surrogate derivative.}
For the derivative, the stop-gradient operator treats both \(\operatorname{sg}\!\bigl(\mathrm{clip}(x,\ell,u)\bigr)\) and \(\operatorname{sg}(x)\) as constants during backpropagation. Hence the derivative of the clipped factor itself is not used, including at the ordinary nondifferentiable points \(x=\ell\) and \(x=u\). We have
\[
\begin{aligned}
D^{\mathrm{sg}} s(x)
&=\operatorname{sg}\!\bigl(\mathrm{clip}(x,\ell,u)\bigr)\cdot
D^{\mathrm{sg}}\!\left(\frac{x}{\operatorname{sg}(x)}\right) \\
&=\mathrm{clip}(x,\ell,u)\cdot \frac{1}{x} \\
&=\frac{\mathrm{clip}(x,\ell,u)}{x}.
\end{aligned}
\]
Because \(x>0\), this is equivalently
\[
D^{\mathrm{sg}}s(x)=
\begin{cases}
\ell/x, & 0<x<\ell,\\
1, & \ell\le x\le u,\\
u/x, & x>u,
\end{cases}
\]
and also \(D^{\mathrm{sg}}s(x)=\mathrm{clip}(1,\ell/x,u/x)\).

\smallskip
\noindent\textit{Upper-side saturation.}
Let \(x=x_\theta\), let \(L(\theta)=-A\,s(x_\theta)\) with \(A>0\) treated as constant during differentiation, and suppose \(x_\theta>u\). The derivative identity gives
\[
D^{\mathrm{sg}} s(x_\theta)=\frac{u}{x_\theta}.
\]
By the chain rule,
\[
\nabla_\theta^{\mathrm{sg}} L(\theta)
=-A\,D^{\mathrm{sg}}s(x_\theta)\,\nabla_\theta x_\theta
=-A\,\frac{u}{x_\theta}\,\nabla_\theta x_\theta.
\]
Since \(A>0\), \(u>0\), and \(x_\theta>u\), this gradient is non-zero whenever \(\nabla_\theta x_\theta\neq 0\). For comparison, the standard clipped branch \(L_{\mathrm{clip}}(\theta)=-A\,\mathrm{clip}(x_\theta,\ell,u)\) is constant in \(x_\theta\) on the region \(x_\theta>u\), so its automatic-differentiation gradient is zero there. This proves that GClip preserves a straight-through signal after upper-side clip saturation.

\smallskip
\noindent\textit{Interior agreement.}
Now let \(x_0\in(\ell,u)\). Because \(x_0\) lies strictly inside the clipping interval, \(\mathrm{clip}(x_0,\ell,u)=x_0\). The forward identity from the first claim gives
\[
s(x_0)=\mathrm{clip}(x_0,\ell,u)=x_0.
\]
The derivative identity from the second claim gives
\[
D^{\mathrm{sg}}s(x_0)=\frac{\mathrm{clip}(x_0,\ell,u)}{x_0}=\frac{x_0}{x_0}=1.
\]
Thus GClip agrees exactly with the unclipped map to first order at every interior point of the clipping interval. The proof is complete.
\end{proof}

\begin{corollary}[Selected branches in the clipped surrogate]
Let
\[
\phi_A(x)=-\min(xA,\operatorname{GClip}(x;\ell,u)A)
\]
be the token-level loss branch used in Equation~\eqref{eq:token_loss}, with \(A\) treated as constant during differentiation and away from branch-tie points. If \(A>0\) and \(x>u\), the clipped branch is selected and
\[
\nabla_\theta^{\mathrm{sg}}\phi_A(x_\theta)
=
-A\frac{u}{x_\theta}\nabla_\theta x_\theta.
\]
If \(A<0\) and \(0<x<\ell\), the clipped branch is also selected and
\[
\nabla_\theta^{\mathrm{sg}}\phi_A(x_\theta)
=
-A\frac{\ell}{x_\theta}\nabla_\theta x_\theta.
\]
Therefore, Equation~\eqref{eq:token_loss} preserves nonzero straight-through signal in the upper-side positive-advantage case and in the lower-side negative-advantage case. The first case supports learning positively preferred edited suffixes; the second case continues to penalize lower-ranked tokens whose ratios have moved below the lower clip bound.
\end{corollary}

\begin{proof}
This follows directly from the branch selected by the minimum and the derivative identity \(D^{\mathrm{sg}}\operatorname{GClip}(x;\ell,u)=\mathrm{clip}(x,\ell,u)/x\). Thus GClip leaves the clipped forward scalar value unchanged while replacing the ordinary clipped branch with a nonzero straight-through gradient in the saturation cases selected by the surrogate.
\end{proof}

\section{Edited-Suffix Distribution Mismatch}
\label{app:gclip_mismatch}

The \emph{Standard Clip} ablation in Appendix~\ref{app:ablations} is motivated by the local distribution mismatch introduced when an accepted training-time edit replaces the original continuation after the accepted editable boundary \(\hat{k}\). The edited suffix no longer follows the sampled trajectory under \(\pi_{\theta_{\mathrm{old}}}\), so its early tokens can receive low old-policy probability near the boundary. Figure~\ref{fig:forced_cue_case} shows a representative accepted pair over the first 80 suffix positions; the shaded cue region starts at \(\hat{k}\) and includes the inserted final-answer cue. The original suffix remains on the original sampled trajectory, whereas the edited suffix begins with the verified final-answer string that closes the thinking process, followed by the regenerated content process. This mismatch can make positively preferred edited tokens more likely to enter upper-side clip saturation once their current-policy probability increases, which is the case where GClip preserves straight-through learning signal.

This mismatch also clarifies the status of the objective. The edited suffix is deliberately introduced as a verified synthetic alternative to the original continuation. Consequently, the edited-trajectory update should not be read as satisfying the usual on-policy sampling assumptions behind PPO-style clipping. Instead, DRE uses the clipped-ratio form as a bounded forward surrogate for a preference constructed from verifier-accepted edits.

\begin{figure}[t]
\centering
\includegraphics[width=0.95\linewidth]{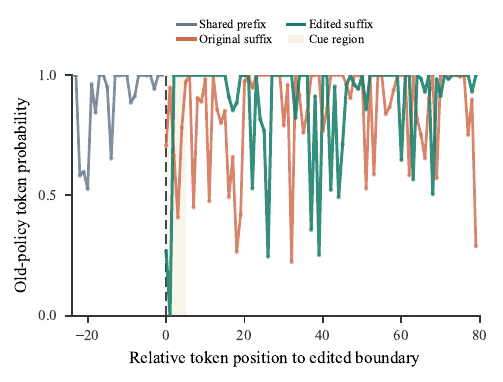}
\caption{Representative old-policy probability mismatch near an accepted editable boundary. The shaded cue region starts at the accepted boundary and includes the inserted final-answer cue. After the shared thinking prefix, the edited suffix often receives low old-policy probability near the boundary and over early edited-suffix tokens, while the original suffix remains on the original sampled trajectory.}
\label{fig:forced_cue_case}
\end{figure}

\section{Representative Rollouts Across Training}
\label{app:qualitative_rollouts}

Figure~\ref{fig:qualitative_rollouts} shows a single prompt across early, middle, and late training steps from the same GRPO training run discussed in Section~\ref{sec:pathology}. The example illustrates the contrast between thinking before and after answer emergence in Figure~\ref{fig:confound}. Thinking before answer emergence can be useful: the middle training step fixes the early training step's failure to convert ``3 hours and 24 minutes'' into minutes before boxing the answer. Thinking after answer emergence can become overthinking: the late training step already boxes the correct answer, then enters an additional verification block and boxes the same answer again. To keep the appendix compact, we show only the final portion of each response after the shared algebraic setup.

\definecolor{rolloutborder}{RGB}{108,128,156}
\definecolor{rolloutouter}{RGB}{240,244,250}
\definecolor{rolloutheader}{RGB}{219,228,240}
\definecolor{rolloutpanelheader}{RGB}{205,218,236}
\definecolor{rolloutpanelbg}{RGB}{251,252,254}
\definecolor{rollouthighlightbg}{RGB}{255,230,230}
\definecolor{rollouthighlightborder}{RGB}{196,48,43}
\definecolor{rollouthighlighttext}{RGB}{156,24,24}
\newlength{\rolloutpanelheight}
\setlength{\rolloutpanelheight}{28.5em}

\newcommand{\rollouthl}[1]{%
\begingroup
\setlength{\fboxsep}{1.2pt}%
\setlength{\fboxrule}{0.8pt}%
\mathchoice
  {\fcolorbox{rollouthighlightborder}{rollouthighlightbg}{$\displaystyle \color{rollouthighlighttext}\mathbf{#1}$}}
  {\fcolorbox{rollouthighlightborder}{rollouthighlightbg}{$\textstyle \color{rollouthighlighttext}\mathbf{#1}$}}
  {\fcolorbox{rollouthighlightborder}{rollouthighlightbg}{$\scriptstyle \color{rollouthighlighttext}\mathbf{#1}$}}
  {\fcolorbox{rollouthighlightborder}{rollouthighlightbg}{$\scriptscriptstyle \color{rollouthighlighttext}\mathbf{#1}$}}
\endgroup}

\newcommand{\rolloutcard}[2]{%
\fcolorbox{rolloutborder}{rolloutpanelbg}{%
\parbox[t][\rolloutpanelheight][t]{0.295\textwidth}{%
\colorbox{rolloutpanelheader}{\parbox{\dimexpr\linewidth-2\fboxsep\relax}{\centering\small\textbf{#1}}}\par
\vspace{0.6em}
\raggedright\footnotesize
\setlength{\abovedisplayskip}{0.35em}
\setlength{\belowdisplayskip}{0.35em}
\setlength{\abovedisplayshortskip}{0.25em}
\setlength{\belowdisplayshortskip}{0.25em}
#2%
}}}

\begin{figure*}[t]
\centering
\fcolorbox{rolloutborder}{rolloutouter}{%
\parbox{0.97\textwidth}{%
\colorbox{rolloutheader}{\parbox{\dimexpr\linewidth-2\fboxsep\relax}{\small\textbf{Prompt.} Every morning Aya goes for a 9-kilometer-long walk and stops at a coffee shop afterwards. When she walks at a constant speed of \(s\) kilometers per hour, the walk takes her 4 hours, including \(t\) minutes spent in the coffee shop. When she walks \(s+2\) kilometers per hour, the walk takes her 2 hours and 24 minutes, including \(t\) minutes spent in the coffee shop. Suppose Aya walks at \(s+\nicefrac{1}{2}\) kilometers per hour. Find the number of minutes the walk takes her, including the \(t\) minutes spent in the coffee shop.}}\par\medskip

\noindent{\small\textit{The three excerpts are taken directly from stored model outputs; \texttt{[earlier steps omitted]} marks omitted prefix text, and line breaks are adjusted for layout.}}\par\medskip

\centering
\begin{tabular}{@{}m{0.31\textwidth}@{\hspace{0.015\textwidth}}m{0.31\textwidth}@{\hspace{0.015\textwidth}}m{0.31\textwidth}@{}}
\rolloutcard
    {Early Training Step}
{
\textit{[earlier steps omitted]}\\[0.25em]
Using the first equation:
\[
\frac{9}{2.5} + \frac{t}{60} = 4
\]
\[
3.6 + \frac{t}{60} = 4
\]
\[
\frac{t}{60} = 0.4
\]
\[
t = 24 \text{ minutes}
\]
\textbf{Step 5: Calculate the total time when Aya walks at \( s + \frac{1}{2} = 3 \) km/h.}
\[
\begin{aligned}
\text{Total time} &= \frac{9}{3} + \frac{24}{60} \\
&= 3 + 0.4 \\
&= 3.4 \text{ hours}
\end{aligned}
\]
Convert 0.4 hours to minutes:
\[
0.4 \times 60 = 24 \text{ minutes}
\]
So, the total time is 3 hours and 24 minutes.

\textbf{Final Answer:}
\[
\boxed{344}
\]
}
&
\rolloutcard
    {Middle Training Step}
{
\textit{[earlier steps omitted]}\\[0.25em]
Given \( s = 2.5 \) km/h,
\[
s + \frac{1}{2} = 2.5 + 0.5 = 3 \text{ km/h}
\]
First, find the walking time at 3 km/h for 9 km:
\[
T_w = \frac{9}{3} = 3 \text{ hours}
\]
Then, the total time including coffee shop time:
\[
T = T_w + t = 3 + 0.4 = 3.4 \text{ hours}
\]
Convert 3.4 hours to hours and minutes:
\[
0.4 \text{ hours} = 0.4 \times 60 = 24 \text{ minutes}
\]
So,
\[
T = 3 \text{ hours and } 24 \text{ minutes}
\]
But the problem asks for the number of minutes. So, convert 3 hours and 24 minutes to minutes:
\[
\begin{aligned}
3 \times 60 + 24 \\
= 180 + 24 \\
= \rollouthl{204} \text{ minutes}
\end{aligned}
\]

\textbf{Final Answer:}
\[
\boxed{\rollouthl{204}}
\]
}
&
\rolloutcard
    {Late Training Step}
{
\textit{[earlier steps omitted]}\\[0.25em]
Thus, the final answer is:
\[
\boxed{\rollouthl{204}}
\]

Wait, let's ensure all steps are accurate.

\textbf{Verification}

Given:
1. \( \frac{9}{s} + \frac{t}{60} = 4 \) (1)
2. \( \frac{9}{s+2} + \frac{t}{60} = 2.4 \) (2)

Found:
\( s = 2.5 \) km/h, \( t = 24 \) minutes

For third scenario:
\begin{itemize}\itemsep0.1em\topsep0.15em\parsep0pt\partopsep0pt\leftmargin1.1em
\item Speed: \( s + \frac{1}{2} = 2.5 + 0.5 = 3 \) km/h
\item Time walking: \( \frac{9}{3} = 3 \) hours
\item Time in coffee shop: \( t = 24 \) minutes \( = 0.4 \) hours
\item Total time: \( 3 + 0.4 = 3.4 \) hours
\item Convert to minutes: \( 3.4 \times 60 = \rollouthl{204} \) minutes
\end{itemize}

Everything seems consistent. Therefore, the answer is:
\[
\boxed{\rollouthl{204}}
\]
}
\end{tabular}
}}\par

\caption{Single-prompt rollout progression across training. Each panel is an abridged excerpt from the stored model output at that training step; omitted text is marked explicitly, and the remaining wording is unchanged apart from markdown-to-\LaTeX{} formatting. The early training step reaches the correct verbal conclusion but boxes \(\boxed{344}\), showing that thinking before the final answer has not yet completed the required format conversion. The middle training step completes the needed conversion before boxing \(\boxed{204}\). The late training step boxes \(\boxed{204}\) and then continues into an additional verification block before boxing the same answer again. This example complements Figure~\ref{fig:confound}: thinking before \(k^\star\) can improve correctness, while thinking after \(k^\star\) can become overthinking attached to successful trajectories.}
\label{fig:qualitative_rollouts}
\end{figure*}

\section{Representative Accepted Edit}
\label{app:accepted_edit_case}

Figure~\ref{fig:accepted_edit_case} shows an accepted edit from the rollout-editing data used for the method-facing analyses. The example is chosen to make the edit operation visible: the two trajectories share a long thinking prefix, diverge after the accepted editable boundary \(\hat{k}\), and the edited branch closes the thinking process with the verified answer before regenerating the content process. The left branch is the lower-scored original continuation, and the right branch is the preferred accepted edit.

\definecolor{drecommonborder}{RGB}{91,111,135}
\definecolor{drecommonbg}{RGB}{245,248,252}
\definecolor{dreorigborder}{RGB}{180,74,68}
\definecolor{dreorigbg}{RGB}{255,246,245}
\definecolor{dreeditborder}{RGB}{58,136,103}
\definecolor{dreeditbg}{RGB}{244,251,247}
\definecolor{dretitlebg}{RGB}{231,237,246}

\begin{figure*}[t]
\centering
\begingroup
\setlength{\fboxsep}{6pt}
\setlength{\fboxrule}{0.8pt}
\fcolorbox{drecommonborder}{drecommonbg}{%
\begin{minipage}{0.94\textwidth}
\footnotesize
\textbf{Shared thinking prefix through \(\hat{k}\).}
\textit{[earlier thinking omitted]} The step-by-step sum with the verified \(a_n\) values gives 450, so I think that's the answer.

\medskip
Just to make 100\% sure, let's compute the sum again in a different order. Starting from \(n=11\) down to \(n=1\), the alternating sum gives \(450\) again after the \(n=3\), \(n=2\), and \(n=1\) terms.

\medskip
Yes! Same result. So regardless of the order, we get 450.
\end{minipage}}

\vspace{0.35em}
\noindent
\begin{minipage}{0.47\textwidth}
\centering
{\color{dreorigborder}\shortstack{\rule{1.6pt}{1.0em}\\[-0.4em]\scalebox{1.8}{\(\blacktriangledown\)}}}
\end{minipage}\hspace{0.02\textwidth}%
\begin{minipage}{0.47\textwidth}
\centering
{\small accepted editable boundary \(\hat{k}\)}\par\vspace{-0.1em}
{\color{dreeditborder}\shortstack{\rule{1.6pt}{1.0em}\\[-0.4em]\scalebox{1.8}{\(\blacktriangledown\)}}}
\end{minipage}
\vspace{0.15em}

\begin{tabular}{@{}p{0.47\textwidth}@{\hspace{0.02\textwidth}}p{0.47\textwidth}@{}}
\fcolorbox{dreorigborder}{dreorigbg}{%
\begin{minipage}[t]{0.43\textwidth}
\footnotesize
\colorbox{dretitlebg}{\parbox{\dimexpr\linewidth-2\fboxsep\relax}{\centering\textbf{Original continuation after \(\hat{k}\)}}}\par
\vspace{0.4em}
I think that's solid. The key was not getting confused by the initial simplification, computing the first few \(a_n\) correctly, finding the pattern for the minima at the median, deriving the formulas, verifying with multiple methods, then carefully summing with the correct signs.

\medskip
The problem says ``for a positive integer \(n\),'' so \(n\) starts at 1, which we did, and the sum is from \(n=1\) to 11, which we covered.

\medskip
I don't see any mistakes in the \(a_n\) values or the summation now, so I think the answer is 450.

\medskip
\texttt{</think>}

\medskip
To solve the problem, we begin by simplifying the function:
\[
f_n(x)=\left(\sum_{i=1}^n |x-i|\right)^2-\sum_{i=1}^n (x-i)^2.
\]
\end{minipage}}
&
\fcolorbox{dreeditborder}{dreeditbg}{%
\begin{minipage}[t]{0.43\textwidth}
\footnotesize
\colorbox{dretitlebg}{\parbox{\dimexpr\linewidth-2\fboxsep\relax}{\centering\textbf{Accepted edit}}}\par
\vspace{0.4em}
\begin{center}
\textbf{Final Answer}\\[-0.15em]
\[
\boxed{450}
\]
\texttt{</think>}
\end{center}

\medskip
To solve the problem, we begin by analyzing the function:
\[
f_n(x)=\left(\sum_{i=1}^n |x-i|\right)^2-\sum_{i=1}^n (x-i)^2.
\]
This expression simplifies significantly because \(|x-i|^2=(x-i)^2\). Expanding the square gives
\[
f_n(x)=2\sum_{1\leq i<j\leq n}|x-i||x-j|.
\]
\end{minipage}}
\end{tabular}
\endgroup
\caption{A concrete accepted edit produced by dynamic rollout editing. The shared thinking prefix is preserved through the accepted editable boundary \(\hat{k}\). The original branch continues with additional post-answer checking, while the accepted edit closes the thinking process with the verified answer and regenerates the content process. During training, the accepted edit is preferred over the original continuation, and prefix masking prevents the lower-scored original branch from directly penalizing the shared prefix.}
\label{fig:accepted_edit_case}
\end{figure*}

\section{Case Study: DRE Reduces Post-Solution Continuation}
\label{app:raw_dre_aime_case}

Figure~\ref{fig:raw_dre_aime_case} presents a paired AIME24 example comparing the Raw Model and the DRE-trained model on the same prompt. Once the verified answer emerges, the Raw Model continues to reflect on and re-check the exposed value, producing a long post-solution continuation. In contrast, the DRE-trained model closes the thinking process shortly after deriving the same value.

\definecolor{rawdreborder}{RGB}{98,112,132}
\definecolor{rawdreouter}{RGB}{244,247,251}
\definecolor{rawdreheader}{RGB}{225,234,245}
\definecolor{rawdreorigbg}{RGB}{255,247,246}
\definecolor{rawdreeditbg}{RGB}{245,252,248}
\definecolor{rawdreomitbg}{RGB}{238,241,245}
\definecolor{rawdrehighlight}{RGB}{34,128,82}
\newlength{\rawdrecaseheight}
\setlength{\rawdrecaseheight}{27.2em}

\newcommand{\rawdrecasecard}[3]{%
\fcolorbox{rawdreborder}{#2}{%
\parbox[t][\rawdrecaseheight][t]{0.445\textwidth}{%
\colorbox{rawdreheader}{\parbox{\dimexpr\linewidth-2\fboxsep\relax}{\centering\small\textbf{#1}}}\par
\vspace{0.55em}
\raggedright\footnotesize
\setlength{\abovedisplayskip}{0.32em}
\setlength{\belowdisplayskip}{0.32em}
\setlength{\abovedisplayshortskip}{0.22em}
\setlength{\belowdisplayshortskip}{0.22em}
#3%
}}}
\newcommand{\rawdrecaseomit}[1]{%
\colorbox{rawdreomitbg}{\parbox{\dimexpr\linewidth-2\fboxsep\relax}{\scriptsize #1}}}

\begin{figure*}[t]
\centering
\fcolorbox{rawdreborder}{rawdreouter}{%
\parbox{0.97\textwidth}{%
\colorbox{rawdreheader}{\parbox{\dimexpr\linewidth-2\fboxsep\relax}{\small\textbf{Prompt.} Jen picks \(4\) distinct numbers from \(\{1,\ldots,10\}\), and \(4\) numbers are randomly chosen from the same set. She wins a prize if at least two of her numbers are among the chosen numbers, and wins the grand prize if all four match. If the conditional probability of the grand prize given any prize is \(m/n\), find \(m+n\).}}\par\medskip

\noindent{\small\textit{The panels summarize the same sample slot in the stored Raw Model and DRE outputs. The long repeated checks are compressed, and notation and line breaks are normalized for layout.}}\par\medskip

\begingroup
\centering
\scriptsize
\setlength{\tabcolsep}{0.55em}
\begin{tabular}{@{}lccc@{}}
\hline
Model & Verified answer & Thinking length & Thinking after answer emergence \\
\hline
Raw Model & \(\boxed{\rollouthl{116}}\) & \(22{,}808\) chars & \(\approx 12{,}200\) chars \\
DRE-trained & \(\boxed{116}\) & \(7{,}238\) chars & \(32\) chars \\
\hline
\end{tabular}
\par
\endgroup
\medskip

\centering
\begin{tabular}{@{}m{0.47\textwidth}@{\hspace{0.02\textwidth}}m{0.47\textwidth}@{}}
\rawdrecasecard
    {Raw Model: Long Post-Answer Checking}
    {rawdreorigbg}
{
\textit{Answer emergence.}\\[-0.35em]
\[
\begin{aligned}
N_{\geq 2}
&=\binom{4}{2}\binom{6}{2}
  +\binom{4}{3}\binom{6}{1}
  +\binom{4}{4}\\
&=90+24+1=115,\\
\Pr(\mathrm{grand}\mid\mathrm{prize})&=\frac{1}{115},\\
m+n&=\rollouthl{116}.
\end{aligned}
\]
\emph{``\(m+n=\rollouthl{116}\)? Wait, but wait, wait a second, is this correct?''}

\medskip
\rawdrecaseomit{The Raw Model then continues for over twelve thousand characters. It reopens the interpretation of ``at least two,'' recomputes the match counts, checks small examples and lottery variants, and repeatedly returns to the same value.}

\medskip
\textit{Late continuation.}\\[-0.35em]
\[
\binom{4}{2}\binom{6}{2}=90,\quad
\binom{4}{3}\binom{6}{1}=24,\quad
\binom{4}{4}=1.
\]
\emph{``So the answer should be \(1+115=\rollouthl{116}\).''}
\[
\boxed{\rollouthl{116}}
\]
}
&
\rawdrecasecard
    {DRE-Trained Model: Immediate Closure}
    {rawdreeditbg}
{
\textit{Near the end of thinking.}\\[0.25em]
For exactly \(k\) matches, choose \(k\) numbers from Jen's ticket and \(4-k\) from the other six numbers.

\[
\binom{4}{2}\binom{6}{2}=90,\quad
\binom{4}{3}\binom{6}{1}=24,\quad
\binom{4}{4}=1.
\]
\[
N_{\geq 2}=90+24+1=115.
\]
Only \(1\) of these \(115\) cases is the grand prize, so
\[
\Pr(\mathrm{grand}\mid\mathrm{prize})=\frac{1}{115}.
\]
\[
m=1,\quad n=115,
\qquad
m+n=\textcolor{rawdrehighlight}{116}.
\]

\medskip
\begin{center}
\textbf{Final Answer}\\[0.10em]
\(\boxed{116}\)\\[0.05em]
\texttt{</think>}
\end{center}

\smallskip
\begin{center}
\scriptsize
\begin{minipage}{0.92\linewidth}
The number of prize-winning draws is \(90+24+1=115\), and only \(1\) of these draws is the grand prize. Hence
\[
P(\text{Grand Prize}\mid\text{Prize})=\frac{1}{115}.
\]
Thus, \(m=1\), \(n=115\), and \(m+n=116\).
\[
\boxed{116}
\]
\end{minipage}
\end{center}
}
\end{tabular}
}}\par

\caption{Paired Raw Model and DRE-trained outputs for an AIME24 prompt where both models reach the verified answer \(116\). The Raw Model exposes the final requested value and then continues with an extended post-answer checking phase, including repeated reinterpretation and consistency checks. The DRE-trained model performs the necessary counting calculation and closes the thinking process almost immediately after answer emergence. This qualitative case illustrates the intended effect of DRE: reducing overthinking after answer emergence without changing a correct final answer.}
\label{fig:raw_dre_aime_case}
\end{figure*}

\section{Semantic Dynamics Under DRE}
\label{app:semantic_dre_dynamics}

Figure~\ref{fig:semantic_original_vs_dre_dynamics} gives a representative semantic-space analysis of the same target behavior. Following the reasoning-dynamics analysis of LLM overthinking~\citep{wei2025evolutionthought}, we examine whether DRE shortens post-solution continuation after answer emergence without erasing thinking that may be needed before the answer-emergence boundary \(k^\star\). For each thinking prefix \(k\), we force the thinking process to stop at that prefix, sample content continuations, embed the sampled content processes, and average them into a prefix-level centroid \(e_k\). The left panel projects these centroids into a shared two-dimensional PCA space for visualization only. The right panel reports the high-dimensional semantic residual \(1-\cos(e_k,\bar e_{\geq k^\star})\), where \(\bar e_{\geq k^\star}\) is the centroid of the post-\(k^\star\) prefix embeddings for the corresponding trajectory.

The original model continues through a longer sequence of prefix states and moves within a broad post-\(k^\star\) semantic neighborhood. In contrast, the DRE-trained model reaches a compact region over a shorter prefix range. Because each residual curve uses its own post-\(k^\star\) centroid, the residual values should be read as within-trajectory concentration rather than distances to a shared semantic target. This qualitative analysis is consistent with the main evidence in Section~\ref{sec:experiments}: DRE reduces continuation after answer emergence while preserving pre-\(k^\star\) reasoning that may be needed to expose the verified answer.

\begin{figure*}[t]
\centering
\includegraphics[width=0.96\textwidth]{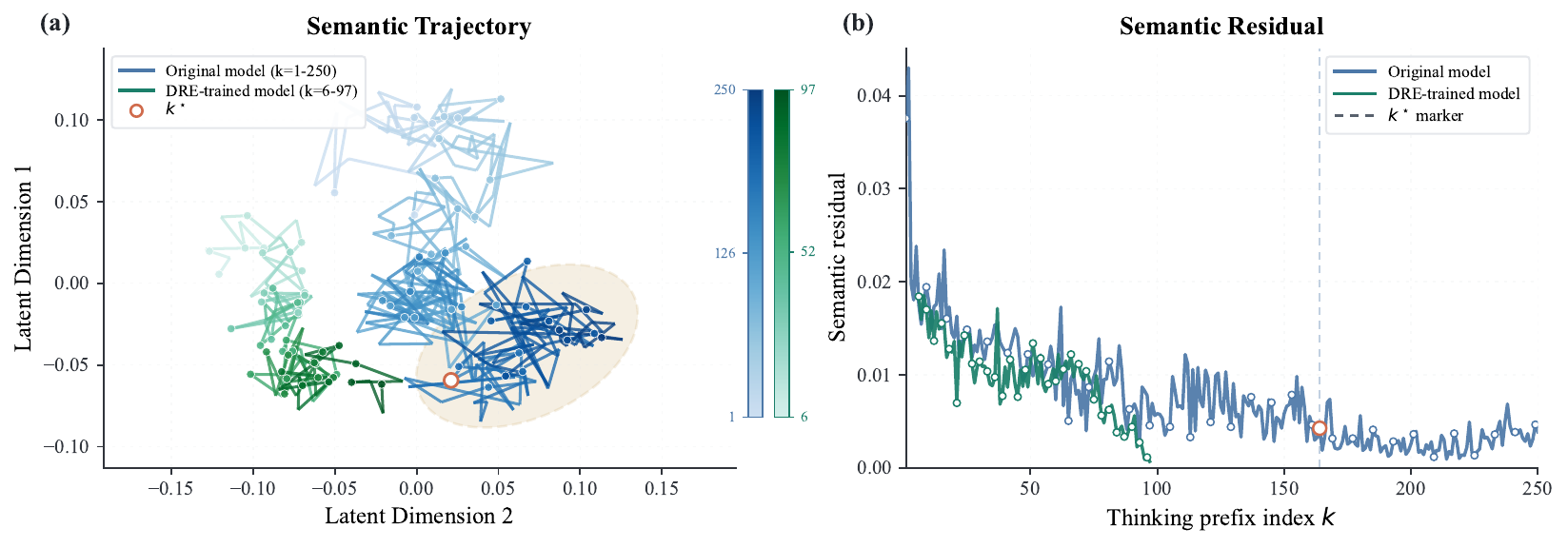}
\caption{Supplementary representative semantic-dynamics analysis for the original model and the DRE-trained model. \textbf{(a)} Prefix-level content centroids are projected with a shared PCA basis; color intensity encodes the thinking prefix index \(k\) separately for each trajectory. The red marker indicates the answer-emergence boundary \(k^\star\) for the original model, and the shaded region summarizes the original model's post-\(k^\star\) neighborhood. \textbf{(b)} Semantic residual is computed in the original embedding space as distance to each trajectory's own post-\(k^\star\) centroid, so it measures within-trajectory concentration rather than distance to a shared target. The DRE-trained trajectory covers fewer prefix states and reaches a low-residual region within its shorter prefix range, while the original model continues with post-\(k^\star\) semantic movement.}
\label{fig:semantic_original_vs_dre_dynamics}
\end{figure*}

\section{Declaration of AI Use}
\label{sec:ai_declaration}

\paragraph{Writing assistance.}
Following the ACL policy on AI writing assistance, we disclose that ChatGPT was used only for language polishing during manuscript preparation. Its use was limited to editing author-written text for grammar, clarity, and readability. It was not used to develop the method, design experiments, formulate scientific claims, interpret results, or produce experimental data. The authors reviewed all AI-assisted edits and remain responsible for the final manuscript.

\paragraph{Experimental context.}
The Qwen3-family models and DeepSeek-R1-style RL setup discussed in this work are part of the research subject and experimental context. These models and settings were not used as writing assistants for drafting the manuscript.

\end{document}